\documentclass{article}

\usepackage{microtype}
\usepackage{graphicx}
\usepackage{subfigure}
\usepackage{booktabs} %
\usepackage{multirow}
\usepackage{amsfonts,amsmath,amssymb,amsthm,bm}
\usepackage{import}
\usepackage{bbm}
\usepackage{nicefrac}       %
\usepackage{xcolor}
\usepackage{float}
\usepackage{svg}
\DeclareUnicodeCharacter{2212}{}
\DeclareUnicodeCharacter{0394}{$\Delta$}
\DeclareUnicodeCharacter{03B5}{$\eps$}
\usepackage{hyperref}

\newcommand{\eps}{\varepsilon}
\newcommand{\pluseq}{\mathrel{+}=}

\newtheorem{definition}{Definition}

\newtheorem{theorem}{Theorem}

\graphicspath{ {./images/} }
\usepackage[accepted]{icml2021}

\icmltitlerunning{Differentially Private EBMs}

\begin{document}

\twocolumn[
\icmltitle{Accuracy, Interpretability, and Differential Privacy via Explainable Boosting}

\icmlsetsymbol{equal}

\begin{icmlauthorlist}
\icmlauthor{Harsha Nori}{msft}
\icmlauthor{Rich Caruana}{msft}
\icmlauthor{Zhiqi Bu}{penn}
\icmlauthor{Judy Hanwen Shen}{stanf}
\icmlauthor{Janardhan Kulkarni}{msft}
\end{icmlauthorlist}

\icmlaffiliation{msft}{Microsoft, Redmond, USA.}
\icmlaffiliation{penn}{University of Pennsylvania, Philadelphia, USA.}
\icmlaffiliation{stanf}{Stanford University, Palo Alto, USA}

\icmlcorrespondingauthor{Harsha Nori}{hanori@microsoft.com}

\icmlkeywords{Machine Learning, ICML}

\vskip 0.3in
]

\printAffiliationsAndNotice{}  %

\addtolength{\textfloatsep}{-0.2in}

\begin{abstract}

We show that adding differential privacy to Explainable Boosting Machines (EBMs), a recent method for training interpretable ML models, yields state-of-the-art accuracy while protecting privacy. Our experiments on multiple classification and regression datasets show that DP-EBM models suffer surprisingly little accuracy loss even with strong differential privacy guarantees. In addition to high accuracy, two other benefits of applying DP to EBMs are: a) trained models provide exact global and local interpretability, which is often important in settings where differential privacy is needed; and b) the models can be edited after training without loss of privacy to correct errors which DP noise may have introduced.

\end{abstract}

\section{Introduction}
\label{introduction}
Security researchers have repeatedly shown that machine learning models can leak information about training data \cite{carlini2018secret, melis2019exploiting}. In industries like healthcare, finance and criminal justice, models are trained on sensitive information, and this form of leakage can be especially disastrous. To combat this, researchers have embraced differential privacy, which establishes a strong mathematical standard for privacy guarantees on algorithms \cite{dwork2006calibrating, dwork2014algorithmic}. In many of these high-stakes situations, model interpretability is also important to provide audits, help domain experts such as doctors vet the models, and to correct unwanted errors before deployment \cite{caruana2015intelligible, rudin2019stop}.
In this paper, we address both concerns by developing a private algorithm for learning Generalized Additive Models (GAMs) \cite{hastie1990generalized}. We show that this method can provide strong privacy guarantees, high accuracy, and exact global and local interpretability on tabular datasets.

While GAMs were traditionally fit using smooth low-order splines \cite{hastie1990generalized}, we focus on Explainable Boosting Machines (EBMs), a modern implementation that learns shape functions using boosted decision trees \cite{lou2012intelligible, nori2019interpretml}. EBMs are especially interesting because they often match the accuracy of complex blackbox algorithms like XGBoost and random forests, 
while having a simple optimization procedure and final structure \cite{chang2020interpretable, wang2020pursuit}. 

Our main contributions for this paper are:
\begin{itemize}
    \item We introduce DP-EBMs, a differentially private version of EBMs, and provide a rigorous privacy analysis of this algorithm using the recently introduced GDP framework \cite{dong2019gaussian}. 
    \item Our experimental results on tabular classification and regression problems show that DP-EBMs significantly outperform other DP learning methods. For example, at $\eps=0.5$, DP-EBMs have at most a 0.05 loss in AUC compared to non-private EBMs on benchmark datasets. 
    \item We demonstrate how combining interpretability with differential privacy can address common concerns with DP in practice by enabling users to repair some of the impact of noise on the model and enforce desirable constraints like monotonicity. 
\end{itemize}

Before diving into details in the following sections, Figure \ref{fig:graph-results} provides a quick peek at the empirical results. In summary, DP-EBMs outperform other differentially private learning methods on a variety of classification and regression tasks for many reasonable values of $\varepsilon$.

\begin{figure*}[ht]
\centering
\includegraphics[width=\textwidth]{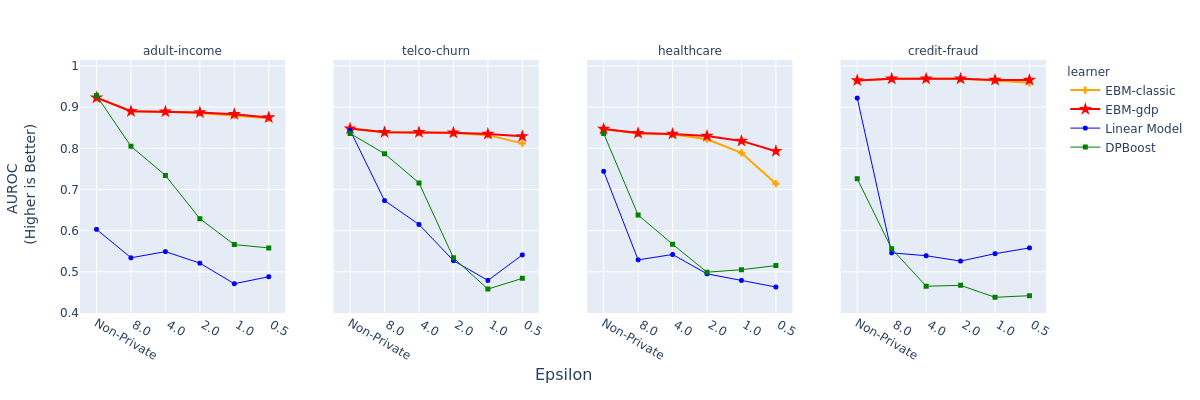}
\label{fig:graph-results}
\vspace{-0.1in}
\caption{Comparison of two variants of DP-EBMs with other DP algorithms on four classification datasets. DP-EBMs significantly outperform other methods in every setting. For the full experimental setup and comparisons on regression models, please see Section 4.}

\end{figure*}

\section{Preliminaries}
\label{preliminaries}

\subsection{Explainable Boosting Machines}
Explainable Boosting Machines belong to the family of Generalized Additive Models (GAMs), which are restricted machine learning models that have the form: \,

\centerline{$g(E[y]) = \beta + f_{0}(x_{0}) + f_{1}(x_{1}) + ... f_{k}(x_{k})$}

where $\beta$ is an intercept, each $f_{j}$ is a univariate function that operates on a single input feature $x_{j}$, and $g$ is a link function that adapts the model to different settings like classification and regression \cite{hastie1990generalized}. 

While GAMs are more flexible than linear models (where each function $f_{j}$ is further restricted to be linear), they are significantly less flexible than most machine learning models due to their inability to learn high-order interactions between features (e.g. $f(x_0, x_1, x_2)$). This restricted additive structure has the benefit of allowing GAMs to provide exact interpretability. At prediction time, each feature contributes a score, which are then summed and passed through a link function. These scores show exactly what each feature contributes to the prediction, and can be sorted, compared, and reasoned about \cite{lundberg2017unified}. In addition, each function $f_k$ can be visualized to provide an exact global description of how the model operates across varying inputs.

EBMs are a recent, popular open-source implementation of boosted-tree GAMs \cite{nori2019interpretml}. We extend the EBM package to include DP-EBMs\footnote{\url{https://github.com/interpretml/interpret}}, which makes DP-EBMs as easy to use as regular EBMs or any scikit-learn model. 

The EBM training procedure begins by bucketing data from continuous features into discrete bins, ensuring that each bin has approximately equal amounts of data. This pre-processing step is a common optimization in tree-based learning algorithms, and is used by popular packages like LightGBM and XGBoost \cite{ke2017lightgbm, chen2016xgboost}. The most time consuming part of training a decision tree is finding the best split; discretizing the data before growing trees reduces the search space for splits which can significantly speed up learning with little cost in accuracy.

After pre-processing, the goal is to learn shape functions $f_k$ for each feature.
In traditional boosting, each tree greedily searches the feature space for the next best feature to split. In contrast, EBMs use \emph{cyclic} gradient boosting to visit each feature in round-robin fashion.
To enforce additivity, each tree is only allowed to use one feature, thus preventing interactions from being learned \cite{lou2012intelligible}.

Cyclic boosting begins by growing a shallow decision tree on the first feature in the dataset. 
The predictions the tree makes on each bin of the histogram are then multiplied by a low learning rate, and these weak predictions for each bin of data become the initial shape function for the first feature. 
The process then iterates to the second feature, where a tree is trained to predict on the residuals (the remaining error) of the first feature's model. Once a shallow tree has been learned for every feature, the boosting process cycles back to the first feature and continues in a round robin fashion for all $E$ epochs to jointly optimize all functions. 
The pseudocode for this algorithm can be found in Algorithm \ref{alg:ebm}.

\begin{figure}
    \centering
    \includegraphics[width=7.55cm]{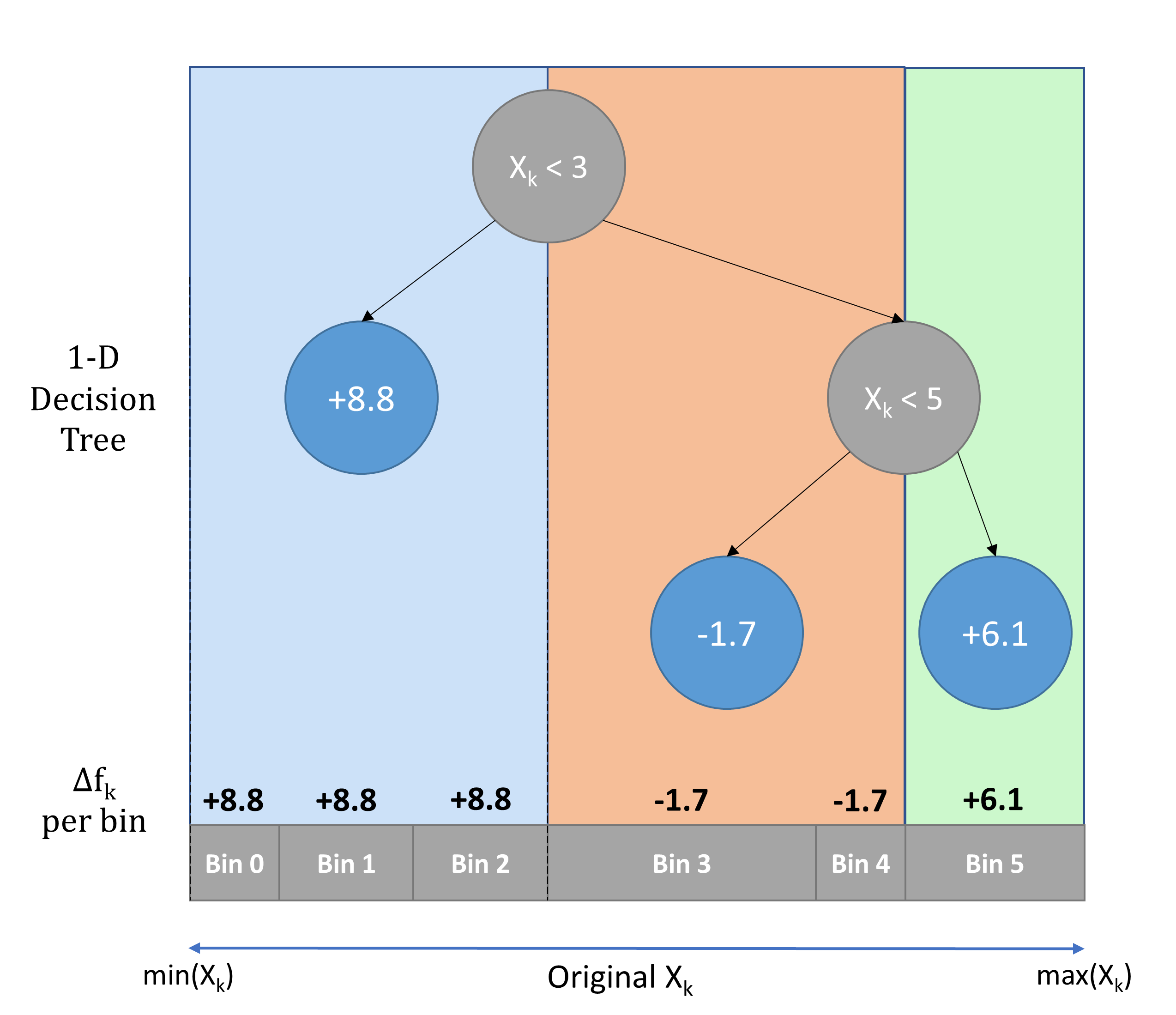}
    \caption{A single iteration of cyclic boosting, showing how each tree operates on pre-processed data and gets collapsed into the univariate shape function $f_k$.}
    \label{fig:ebm-overview}
\end{figure}

\begin{algorithm}[H]
   \caption{Explainable Boosting}
   \label{alg:ebm}
\begin{algorithmic}[1]
   \STATE {\bfseries Input:} data $X$, labels $y$, epochs $E$, 
   
   learning rate $\eta$, max splits $m$
   \STATE {\bfseries Output:} 1d functions $f_k$ per feature
   \STATE
   \STATE $t = 0$
   \STATE Initialize residuals: $r^t_i = y_i$
   \FOR{feature $0...K$} 
    \STATE Bin data: $H_k = Bin(X[:, k])$
    \STATE Initialize output function: $f^t_k = [0,...,0]$
   \ENDFOR
   \STATE
   
\FOR{epoch $1...,E$}
   \FOR{feature $0,...,K$}
        \STATE $t \pluseq 1$
        \STATE Select best splits $S_0,...,S_m$
        \FOR{split $\ell \in  \{0,...,m\}$}
            \STATE Sum residuals: $T = \eta \cdot \sum_{b \in S_{\ell}} \sum_{x_i \in H_k (b)} {r^t_i}$
            \STATE \textcolor{white}{.}...
            \vspace{0.09in}
            \STATE Calculate average: $\mu = \frac{T}{\sum_{b \in S_\ell} H_k(b)}$
              \FOR{each histogram bin { $b \in S_\ell$}}
                    \STATE  Update output function: $f^t_k(b) = f^t_k(b) + \mu$
                    \STATE \textcolor{white}{.}...
              \ENDFOR
        \ENDFOR
        \FOR{each data point $x_i$}
            \STATE Residuals: $r^{t+1}_i = y_i - \sum_{k}{f^t_k(\rho(H_k,x_i))}$
        \ENDFOR
    \ENDFOR
\ENDFOR
\end{algorithmic}
\end{algorithm}

\begin{algorithm}[t]
   \caption{Differentially Private Explainable Boosting}
   \label{alg:dp-ebm}
\begin{algorithmic}[1]
   \STATE {\bfseries Input:} data $X$, labels $y$, epochs $E$, learning rate $\eta$, max splits $m$, \textcolor{blue}{range of labels $R$, privacy parameters $\eps, \delta$}
   \STATE {\bfseries Output:} 1d functions $f_k$ per feature
   \STATE
   \STATE $t = 0$
   \STATE Initialize residuals: $r^t_i = y_i$
   \FOR{feature $0...K$} 
    \STATE \textcolor{blue}{Privately bin data}: $\hat{H}_k = DPBin(X[:, k], \eps_{bin})$
    \STATE Initialize output function: $f^t_k = [0,...,0]$
   \ENDFOR
   \STATE
   
\FOR{epoch $1...,E$}
   \FOR{feature $0,...,K$}
        \STATE $t \pluseq 1$
        \STATE \textcolor{blue}{Randomly select splits} $S_0,...,S_m$
        \FOR{split $\ell \in  \{0,...,m\}$}
       \STATE Sum residuals: $T = \eta \cdot \sum_{b \in S_{\ell}} \sum_{x_i \in \hat{H}_k (b)} {r^t_i}$
        \STATE \textcolor{blue}{Add noise}: $\hat{T} = T +  \sigma\cdot \eta R\cdot\mathcal{N}(0,1)$
        \STATE \textcolor{blue}{Calculate private average}: $\mu = \frac{\hat{T}}{\sum_{b \in S_\ell} \hat{H}_k(b)}$
              \FOR{each histogram bin { $b \in S_\ell$}}
                    \STATE  Update output function: $f^t_k(b) = f^t_k(b) + \mu$
                    \STATE  We release $f^t_k(b)$ values publicly.   
              \ENDFOR
        \ENDFOR
        \FOR{each data point $x_i$}
            \STATE Residuals: $r^{t+1}_i = y_i - \sum_{k}{f^t_k(\rho(\hat{H}_k,x_i))}$  %
        \ENDFOR
    \ENDFOR
\ENDFOR
\end{algorithmic}
\end{algorithm}

\subsection{Differential Privacy}

Here we state some basic results in Differential Privacy (DP) that we use in our analysis.
\begin{definition}[Differential Privacy]
	A randomized algorithm $\mathcal{A}$ is  ($\eps$,$\delta$)-differentially private if for all neighboring databases $D_1$, $D_2$ $\in D^n$, and for all sets $\mathcal{S}$ of possible outputs: 
\begin{equation}
\Pr[\mathcal{A}(D_1) \in \mathcal{S}] \leq e^{\eps}\Pr[\mathcal{A}(D_2) \in \mathcal{S}] +\delta
\end{equation}
\end{definition}

\begin{theorem}[Gaussian Mechanism \cite{dwork2014algorithmic}]
\label{lem:gaussian mechanism}
Given any function $f: D \rightarrow \mathbb{R}^k$, the Gaussian Mechanism is defined as: 
\begin{equation}
	\mathcal{M}(x, f(.), \eps, \delta) = f(x) + (Y_i, ..., Y_k)
\end{equation}
where $\Delta_2$ is the $\ell_2$-sensitivity and $Y_i$ are i.i.d. random variables drawn from $\mathcal{N}(0, \sigma^2)$. The Gaussian Mechanism is ($\eps, \delta$)-differentially private when $\sigma > \sqrt{2\ln1.25/\delta}\Delta_2/\eps$ and $\eps \in (0, 1)$. 
\end{theorem}

One of the main strengths of DP are the composition theorems, which analyze the cumulative privacy guarantee when applying many differentially private mechanisms.
\smallskip
\begin{theorem}{(Theorem 4.3 from \cite{kairouz2017composition})}

For $\Delta>0$, $\eps > 0$ and $\delta \in [0, 1]$, the mechanism that adds Gaussian noise with variance:
    \begin{equation}
        8k\Delta^2 log(e + (\eps / \delta )) / \eps^2
    \end{equation}
satisfies $(\eps,\delta)$-differential privacy under $k$-fold adaptive composition.

\label{prop:k-comp}
\end{theorem}

A qualitative way to understand the above theorem is that if there are $k$ differentially private mechanisms each of which is $(\varepsilon, \delta)$-DP acting on the same data set, then the overall privacy loss is roughly $\varepsilon \cdot \sqrt k$.

Unfortunately, Theorem \ref{prop:k-comp} is not the optimal bound one can achieve on the composition of private mechanisms.
A tighter analysis of composition for Gaussian mechanisms, called Gaussian Differential Privacy (GDP), was recently proposed by \cite{dong2019gaussian}.
In our experiments, GDP analysis gave us better privacy bounds. We summarize the main theorems we borrow from \cite{dong2019gaussian} below. For completeness, we compare the results from both composition methods (``EBM-Classic" and ``EBM-GDP") in Tables \ref{classification-table} and \ref{regression-table}.

\smallskip
\smallskip
\begin{theorem}
For a dataset $D$, define the Gaussian mechanism that operates on a univariate statistic $\theta$ with sensitivity $\Delta$ as $M(D) = \theta(D) + Noise$, where $\textit{Noise}$ is sampled from a Gaussian distribution  $\mathcal{N}(0, \Delta^2/\mu^2)$. Then, $M$ is $\mu$-GDP.
\label{thm:GDPdef}
\end{theorem}

If $M_1, M_2, \ldots, M_k$ are $k$ GDP mechanisms with parameters $\mu_1, \mu_2, \ldots, \mu_k$, then the following GDP composition theorem holds:

\begin{theorem}
The $k$-fold composition of $\mu_i$-GDP mechanisms is $\sqrt{\mu^2_1 + \mu^2_2 + \ldots \mu^2_k}$-GDP.
\label{thm:GDPComp}
\end{theorem}

Finally, one can convert GDP guarantees to the standard $(\epsilon, \delta)$-DP guarantee using the following theorem:
\begin{theorem}
A mechanism is $\mu$-GDP if and only if it is $(\eps, \delta)$-DP where
\begin{equation*}
\label{eq:main}
\delta = \phi(-\frac{\eps}{\mu} + \frac{\mu}{2}) - e^{\eps} \phi(-\frac{\eps}{\mu} - \frac{\mu}{2}) 
\end{equation*}
\label{thmGDPtoDP}
\end{theorem}

Besides the mathematical elegance, a key advantage of GDP is that it provides a tighter analysis of composition guarantees of differentially private mechanisms.

\subsection{Notation}
For rest of the paper, we adopt the following notation. 
We denote by $H_k$ the histogram for feature $k$ and $\hat{H}_k$ for the corresponding differentially private histogram. 
We use $K$ to denote the total number of features.
By a slight abuse of notation, we write $x_i \in H_k(b)$ to mean that the data point $x_i$ belongs to the histogram bin $b$, and use $\rho(H_k, x_i)$ to look up the bin $b$ such that $x_i \in H_k(b)$.

\section{Algorithms}

To add differential privacy guarantees to EBMs, we modify the EBM training procedure Algorithm \ref{alg:ebm}, yielding DP-EBM Algorithm \ref{alg:dp-ebm}.
We first modify the pre-processing procedure to generate differentially private bins (published as histograms $\hat{H}_k$ per feature), which log the bin ranges and how many data points fall in each bin (line 7).
Next, we analyze the boosting process. 
In traditional tree building, there are two major data-intensive operations: learning the structure of the tree (what feature and feature threshold to install at each node in the tree), and calculating the predicted value of each leaf node \cite{breiman1984classification}. 
Prior work on differentially private tree learning typically splits budget between choosing which features to split on, where to split them, and learning prediction values for each leaf node \cite{fletcher2015differentially, wangscalable, li2020privacy}. 

EBMs naturally avoid spending any privacy budget on choosing which features to include in each tree -- the ``round-robin" schedule of visiting features is completely data agnostic. Furthermore, by choosing the splitting thresholds at random, we can learn the entire structure of each tree without looking at any training data (line 14). 
Prior work and our empirical evaluations both show that choosing random splits results in little accuracy loss \cite{geurts2006extremely, fletcher2019decision}.
We therefore spend the entirety of our budget per iteration on learning the values for each leaf node, which are simply averages of the residuals for the data belonging to each node (lines 16-18).
For each leaf, we sum the residuals of data belonging to that leaf, add calibrated Gaussian noise based on their bounded sensitivity $R$, and divide by a differentially private count of data in each leaf (contained in the previously published $\hat{H}_k$).
Then, as in non-private EBMs, the noisy tree is merged into the feature function $f_k$ (line 20), and the cyclic boosting procedure moves onto the next feature and continues for all $E$ epochs. 
The pseudocode for the DP-EBM algorithm is described in Algorithm \ref{alg:dp-ebm}, with modifications to the non-private version highlighted in blue.

We now provide the privacy analysis of our algorithm using the GDP framework.
Our proof of privacy has the following two components.
First we fix a single iteration of the algorithm (lines 13-26), and show that each iteration is $\frac{1}{\sigma}-GDP$.
At the end of each iteration, we publicly release the functions $f^t_k$ for all $k$.
Note that although the final model only uses the $f^t_k$ values of the {\em last iteration}, releasing every $f^t_k$ leads to a simpler privacy analysis.
Next, we calculate the total privacy loss of our algorithm by simply viewing it as a composition of $E \cdot K$ private mechanisms.
It is important to note that composition theorems work even when the mechanisms depend on each other.

\begin{theorem}
Each iteration of our algorithm is $\frac{1}{\sigma}-GDP$.
\end{theorem}
\begin{proof}
We observe that calculating $T$ in the line 16 is the only step of our algorithm where we access the sensitive information of the users. 
Thus to prove the theorem we need to argue that the noise we are adding satisfies requirements of Theorem \ref{thm:GDPdef}.
Consider
\begin{eqnarray*}
T &=& \eta \cdot \sum_{b \in S_{\ell}} \sum_{x_i \in \hat{H}_k (b)} {r^{t}_i} \\
  &=& \eta \cdot \sum_{b \in S_{\ell}} \sum_{x_i \in \hat{H}_k (b)} \left( y_i - \sum_{k}{f^{t-1}_k(\rho(\hat{H}_k,x_i))} \right)\\ 
  &=& \eta \cdot \left(\sum_{b \in S_{\ell}} \sum_{x_i \in \hat{H}_k (b)} y_i \right)  - Z\\ 
\end{eqnarray*}

The second equality follows from the definition of  $r^t_i$ as given in the line 25 of the algorithm.
Further, $Z$ is computed using publicly released $f^{t-1}_k$ values from the iteration $t-1$, and hence does not depend on the user data.
Therefore, the amount of noise we need to add to the statistic $T$ depends on the sensitivity of quantity $\left(\sum_{b \in S_{\ell}} \sum_{x_i \in \hat{H}_k (b)} y_i \right)$, which we argue is bounded by at most $R$. 
This follows from three simple facts: 1) The range of each $y_i$ is bounded by at most $R$; 2) For each feature, each user contributes exactly to one bin of the histogram; 3) Random splits performed in line 14 of our algorithm partition the histogram bins into disjoint splits, hence each users' data belongs to precisely one split.  
The proof now follows from Theorem \ref{thm:GDPdef}.
\end{proof}

\begin{theorem}
Our algorithm is $\frac{\sqrt {E \cdot K}}{\sigma}-GDP$.
\end{theorem}
\begin{proof}
As each iteration of our algorithm is $\frac{1}{\sigma}-GDP$, the proof follows from the composition of GDP-mechanisms as given in Theorem \ref{thm:GDPComp} over all $E \cdot K$ iterations.
\end{proof}

The GDP bounds can be converted into $(\epsilon, \delta)$-DP guarantees using Theorem \ref{thmGDPtoDP}.
To calibrate $\sigma$ in line 17, we fix the $\eps$ and $\delta$ privacy parameters we want to achieve, use Theorem \ref{thmGDPtoDP} to calculate $\mu$, and finally calculate $\sigma$ by setting $\mu$ = $\frac{\sqrt{E \cdot K}}{\sigma}$.

\section{Experiments}

\begin{figure*}[!t]
\centering
    \includegraphics[width=\textwidth]{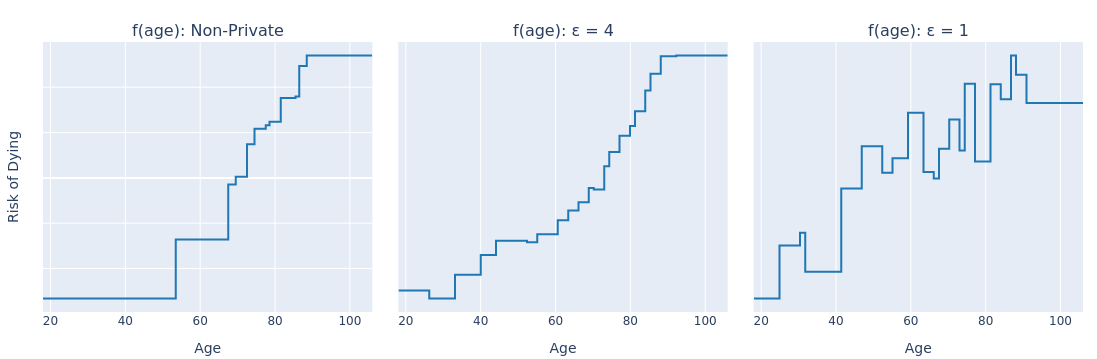}
  \caption{Risk of dying as a function of age from three EBMs trained on the healthcare dataset with varying privacy guarantees.}
  \label{fig:age-compare}
\end{figure*}

We compare the following algorithms on four classification and four regression datasets:

\begin{itemize}
    \item DP-EBM (Algorithm \ref{alg:dp-ebm}): We use the following (default) parameters for all experiments: max\_bins = 32, learning\_rate = 0.01, n\_epochs = 300, max\_leaves = 3, with 10\% of the total privacy budget allocated to binning and 90\% to training. We present two results for DP-EBMs: ``EBM-Classic", where we apply strong composition from \cite{kairouz2017composition}, and ``EBM-GDP", where composition is more tightly tracked via Gaussian Differential Privacy \cite{dong2019gaussian}.
    \item Generalized Linear Models: Linear and Logistic Regression are widely used methods for interpretable machine learning. For both models, we use IBM's differential privacy library \cite{ibm2019diffprivlib} which follows the algorithms described in \cite{sheffet2015private, imtiaz2016symmetric} for linear regression and in \cite{chaudhuri2011differentially} for logistic regression.
    \item DP Boost: DPBoost is a differentially private gradient boosted decision tree algorithm introduced by \cite{li2020privacy}. DPBoost builds on top of LightGBM, a popular tree-based boosting framework \cite{ke2017lightgbm}.
\end{itemize}

To evaluate performance, we generate 25 randomly drawn 80/20 train-test splits and report the average test-set accuracy and standard deviation at varying $\varepsilon$ and fixed $\delta = 10^{-6}$. Results are presented  in Tables \ref{classification-table} and \ref{regression-table} using root mean squared error (RMSE) as the metric for regression and area under the ROC curve (AUROC) for classification.

All models were trained using default or recommended parameters from the literature or open source repositories. Hyperparameter tuning is a privacy-sensitive operation, and how to best partition budget and tune parameters of differentially private models is an open research problem \cite{liu2019private, kusner2015differentially}. 
We did not tune hyperparameters to avoid this complexity. 
The default parameters for DP-EBMs appear to work well on a variety of datasets, which helps conserve the privacy budget and also makes DP-EBMs easy to use in practice.

The datasets used in these experiments (with the exception of the healthcare data, which contains real patient data) are publicly available and summarized in Table \ref{table:dataset_statistics}. We include results from the private healthcare dataset to highlight how these models might perform in a high stakes setting where both privacy and interpretability are critical. 

\begin{table}[H]
\caption{Experiment dataset statistics and descriptions.}
\vspace{5pt}
\centering
\begin{small}

\begin{tabular}{c|ccccc}
Dataset & Domain & N & K & Task & \% Pos \\ \toprule
Adult Income \nocite{Dua:2019} & Finance & 32,561 & 14 & Clas & 24.1\% \\
Telco Churn & Business & 7,043 & 19 & Clas & 26.6\% \\
Credit & Finance & 284,807 & 30 & Clas & 0.2\% \\
Healthcare & Medicine & 14,199 & 46 & Clas & 10.9\% \\ \midrule
Cal-Housing & Real Estate & 20,499 & 8 & Reg & --\\
Elevators & Systems & 16,598 & 18 & Reg & --\\
Pol & Business & 15,000 & 49 & Reg & --\\
Wine-Quality & Food & 5,297 & 11 & Reg & -- \\

\label{table:dataset_statistics}

\end{tabular}

\end{small}
\end{table}

\subsection{Discussion of Experimental Results}

As shown in Figure 1 and again in more detail in Tables 2 and 3, DP-EBMs exhibit strong performance across a wide range of $\varepsilon$ values. 
In classification, the difference in AUROC between non-private DP-EBMs and even $\varepsilon = 0.5$ DP-EBMs is at most 0.05, which is a comparatively modest tradeoff for strong privacy guarantees. 
At a more modest $\eps=4$, the average AUROC for DPBoost, Logistic Regression and DP-EBMs is 0.62, 0.56 and 0.88, respectively.
The datasets chosen are not particularly favorable for EBMs -- in the non-private setting, LightGBM outperforms EBMs in over half of our experiments. However, when training with differential privacy, DP-EBMs outperform other models in all 40 experimental settings. In the following sections, we offer some thoughts on why the DP-EBM algorithm might be comparatively well suited for differential privacy. %

\subsubsection{DP-EBMs vs Linear Models}

In the case of DP linear and logistic regression -- the current standard for intelligible and private learning -- we believe the additional accuracy of DP-EBMs might be explained by the extra flexibility of the GAM model class. 
In the non-private setting, the non-linear functions learned by GAMs often result in a boost in accuracy over the linear functions learned by linear models, and this advantage appears to translate to the private setting as well. In addition, the iterative nature of gradient boosting might give DP-EBMs the ability to recover from the impact of noise earlier in training. 

\subsubsection{DP-EBMs vs DPBoost}

Although it may not be surprising that DP-EBMs outperform restricted models such as DP linear regression, it is a little surprising that DP-EBMs outperform DPBoost which is a less restricted model class than DP-EBMs. We believe this might be due to the significant privacy budget savings when learning each tree. Unlike other DP tree-based learning algorithms, DP-EBMs spend no budget learning the structure of each tree, and focus exclusively on learning the best leaf node values. In addition, by growing shallow trees, each leaf often contains a large proportion of the dataset -- with default parameters of 2 random splits, in expectation each leaf contains $\approx\frac{1}{3}$ of the full data. This ensures that the impact of noise in the differentially private average calculated per iteration is dispersed across many training datapoints. In contrast, each tree in LightGBM/DPBoost contains up to 31 leaf nodes, thereby operating on much smaller amounts of data and magnifying the impact of noise. 

\subsubsection{DP-EBM: Classic Composition vs GDP}

We also compare two variants of privacy budget tracking in DP-EBMs: ``EBM-Classic", which uses strong composition from \cite{kairouz2017composition}, and ``EBM-GDP" which uses Gaussian differential privacy recently proposed by \cite{dong2019gaussian}. While we can show analytically that budget tracking with GDP is tighter for our algorithm (and therefore requires less noise for the same privacy loss), it is interesting that the differences in final model accuracy are typically only noticeable in strong privacy settings ($\varepsilon \leq 1$).

\begin{table*}
\caption{Area Under the ROC Curve (AUROC) algorithm comparison on classification datasets. Higher is better.}
\label{classification-table}
\vskip 0.15in
\begin{center}
\begin{small}
\begin{sc}

\begin{tabular}{cccccc}
\toprule
dataset & $\eps$ &     DPBoost     &     Logistic Regression      &  DPEBM-classic      &     DPEBM-gdp     \\
\midrule
\multirow{6}{*}{adult-income} & 0.5   &  0.558 $\pm$ 0.045 &      0.488 $\pm$ 0.125 &  0.873 $\pm$ 0.007 &  \textbf{0.875 $\pm$ 0.005 } \\
         & 1.0   &  0.566 $\pm$ 0.034 &      0.471 $\pm$ 0.111 &  0.880 $\pm$ 0.006 &  \textbf{0.883 $\pm$ 0.005 } \\
         & 2.0   &  0.629 $\pm$ 0.045 &      0.521 $\pm$ 0.109 &  0.886 $\pm$ 0.005 &  \textbf{0.887 $\pm$ 0.004 } \\
         & 4.0   &  0.734 $\pm$ 0.019 &      0.549 $\pm$ 0.068 &  \textbf{0.889 $\pm$ 0.004} &  \textbf{0.889 $\pm$ 0.004 } \\
         & 8.0   &  0.805 $\pm$ 0.011 &      0.534 $\pm$ 0.070 &  \textbf{0.890 $\pm$ 0.004} &  \textbf{0.890 $\pm$ 0.004 } \\
         & \textit{Non-Private} &  \textbf{0.928 $\pm$ 0.003} &   0.603 $\pm$ 0.066 &  0.923 $\pm$ 0.003 &  0.923 $\pm$ 0.003 \\
\midrule
\multirow{6}{*}{credit-fraud} & 0.5   &  0.442 $\pm$ 0.138 &      0.558 $\pm$ 0.076 &  0.959 $\pm$ 0.015 &  \textbf{0.966 $\pm$ 0.012 } \\
         & 1.0   &  0.438 $\pm$ 0.114 &      0.544 $\pm$ 0.135 &  0.965 $\pm$ 0.014 &  \textbf{0.966 $\pm$ 0.013 } \\
         & 2.0   &  0.467 $\pm$ 0.101 &      0.526 $\pm$ 0.118 &  \textbf{0.969 $\pm$ 0.011} &  \textbf{0.969 $\pm$ 0.011 } \\
         & 4.0   &  0.465 $\pm$ 0.142 &      0.539 $\pm$ 0.172 &  \textbf{0.969 $\pm$ 0.011} &  \textbf{0.969 $\pm$ 0.011 } \\
         & 8.0   &  0.556 $\pm$ 0.145 &      0.546 $\pm$ 0.156 &  \textbf{0.969 $\pm$ 0.011} &  \textbf{0.969 $\pm$ 0.011 } \\
         & \textit{Non-Private} &  0.726 $\pm$ 0.099 &      0.922 $\pm$ 0.019 &  \textbf{0.965 $\pm$ 0.011} &  \textbf{0.965 $\pm$ 0.011} \\
\midrule
\multirow{6}{*}{Healthcare} & 0.5   &  0.515 $\pm$ 0.054 &      0.463 $\pm$ 0.081 &  0.714 $\pm$ 0.036 &  \textbf{0.793 $\pm$ 0.018 } \\
         & 1.0   &  0.505 $\pm$ 0.051 &      0.479 $\pm$ 0.071 &  0.789 $\pm$ 0.016 &  \textbf{0.818 $\pm$ 0.012 } \\
         & 2.0   &  0.499 $\pm$ 0.046 &      0.495 $\pm$ 0.081 &  0.822 $\pm$ 0.012 &  \textbf{0.830 $\pm$ 0.011 } \\
         & 4.0   &  0.567 $\pm$ 0.047 &      0.542 $\pm$ 0.038 &  0.834 $\pm$ 0.011 &  \textbf{0.835 $\pm$ 0.010 } \\
         & 8.0   &  0.638 $\pm$ 0.036 &      0.529 $\pm$ 0.048 &  0.836 $\pm$ 0.010 &  \textbf{0.837 $\pm$ 0.010 } \\
         & \textit{Non-Private} &  0.836 $\pm$ 0.011 &      0.744 $\pm$ 0.014 &  \textbf{0.847 $\pm$ 0.010} &  \textbf{0.847 $\pm$ 0.010 } \\
\midrule
\multirow{6}{*}{telco-churn} & 0.5   &  0.484 $\pm$ 0.100 &      0.541 $\pm$ 0.227 &  0.812 $\pm$ 0.020 &  \textbf{0.829 $\pm$ 0.014 } \\
            & 1.0   &  0.458 $\pm$ 0.088 &      0.479 $\pm$ 0.239 &  0.832 $\pm$ 0.013 &  \textbf{0.835 $\pm$ 0.011 } \\
            & 2.0   &  0.534 $\pm$ 0.109 &      0.527 $\pm$ 0.236 &  0.837 $\pm$ 0.010 &  \textbf{0.838 $\pm$ 0.012 } \\
            & 4.0   &  0.716 $\pm$ 0.067 &      0.615 $\pm$ 0.138 &  0.838 $\pm$ 0.011 &  \textbf{0.839 $\pm$ 0.011 } \\
            & 8.0   &  0.787 $\pm$ 0.014 &      0.673 $\pm$ 0.105 &  \textbf{0.839 $\pm$ 0.011} &  \textbf{0.839 $\pm$ 0.011} \\
            & \textit{Non-Private} &  0.836 $\pm$ 0.008 &      0.844 $\pm$ 0.010 &  \textbf{0.848 $\pm$ 0.009} &  \textbf{0.848 $\pm$ 0.009 } \\
\bottomrule
\end{tabular}

\end{sc}
\end{small}
\end{center}
\vskip -0.1in
\end{table*}
\begin{table*}
\caption{Root Mean Squared Error (RMSE) algorithm comparison on regression datasets. Lower is better.}
\label{regression-table}
\vskip 0.15in
\begin{center}
\begin{small}
\begin{sc}

\begin{tabular}{cccccc}
\toprule
dataset & $\eps$ &      DPBoost         &         Linear Regression         &      DPEBM-classic             &         DPEBM-gdp         \\
\midrule
\multirow{6}{*}{cal-housing} & 0.5   &  383072 $\pm$ 41952 &  111967 $\pm$ 1080 &  85652 $\pm$ 2724 &  \textbf{79967 $\pm$ 1929} \\
             & 1.0   &  204277 $\pm$ 19350 &  110241 $\pm$ 1101 &  78527 $\pm$ 1230 &  \textbf{76827 $\pm$ 1470} \\
             & 2.0   &   122494 $\pm$ 7066 &  109518 $\pm$ 1244 &  75491 $\pm$ 1404 &  \textbf{74573 $\pm$ 1152} \\
             & 4.0   &    96336 $\pm$ 3043 &  108882 $\pm$ 1370 &  73967 $\pm$ 1028 &  \textbf{73754 $\pm$ 1022} \\
             & 8.0   &    90029 $\pm$ 2508 &  107815 $\pm$ 1460 &  73327 $\pm$ 1118 &  \textbf{ 73165 $\pm$ 955 }\\
             & \textit{Non-Private} &     \textbf{47007 $\pm$ 885} &   69850 $\pm$ 1164 &   51644 $\pm$ 925 &   51644 $\pm$ 925 \\
\midrule
\multirow{6}{*}{elevators} & 0.5   &           0.051 $\pm$ 0.005 &          4.671 $\pm$ 1.975 &         0.010 $\pm$ 0.001 &  \textbf{0.006 $\pm$ 0.000 } \\
             & 1.0   &           0.025 $\pm$ 0.002 &          2.669 $\pm$ 1.214 &         0.007 $\pm$ 0.000 &         \textbf{0.005 $\pm$ 0.000 } \\
             & 2.0   &           0.013 $\pm$ 0.001 &          1.384 $\pm$ 0.570 &         0.006 $\pm$ 0.000 &         \textbf{0.005 $\pm$ 0.000 } \\
             & 4.0   &           0.008 $\pm$ 0.000 &          0.754 $\pm$ 0.202 &         0.005 $\pm$ 0.000 &         \textbf{0.004 $\pm$ 0.000 } \\
             & 8.0   &           0.006 $\pm$ 0.000 &          0.410 $\pm$ 0.201 &        \textbf{ 0.004 $\pm$ 0.000} &         \textbf{0.004 $\pm$ 0.000 } \\
             & \textit{Non-Private} &           \textbf{0.002 $\pm$ 0.000} &          0.003 $\pm$ 0.000 &         0.004 $\pm$ 0.000 &         0.004 $\pm$ 0.000 \\
\midrule
\multirow{6}{*}{pol} & 0.5   &          78.190 $\pm$ 9.583 &         31.326 $\pm$ 0.418 &        35.156 $\pm$ 1.728 &        \textbf{30.988 $\pm$ 0.962 } \\
             & 1.0   &          50.527 $\pm$ 5.482 &         30.640 $\pm$ 0.288 &        30.911 $\pm$ 1.014 &        \textbf{28.391 $\pm$ 0.585 } \\
             & 2.0   &          47.511 $\pm$ 4.636 &         30.500 $\pm$ 0.248 &        27.616 $\pm$ 0.644 &        \textbf{26.303 $\pm$ 0.561 } \\
             & 4.0   &          45.592 $\pm$ 2.942 &         30.463 $\pm$ 0.256 &        25.454 $\pm$ 0.389 &        \textbf{24.934 $\pm$ 0.332 } \\
             & 8.0   &          45.435 $\pm$ 1.109 &         30.459 $\pm$ 0.258 &        24.625 $\pm$ 0.230 &        \textbf{24.313 $\pm$ 0.237 } \\
             & \textit{Non-Private} &           \textbf{4.703 $\pm$ 0.228} &         30.464 $\pm$ 0.264 &        13.780 $\pm$ 0.667 &        13.780 $\pm$ 0.667 \\
\midrule
\multirow{6}{*}{wine-quality} & 0.5   &           4.647 $\pm$ 0.390 &          3.621 $\pm$ 1.740 &         1.589 $\pm$ 0.132 &         \textbf{0.938 $\pm$ 0.036 } \\
             & 1.0   &           2.151 $\pm$ 0.302 &          2.133 $\pm$ 0.713 &         1.181 $\pm$ 0.074 &         \textbf{0.841 $\pm$ 0.025 } \\
             & 2.0   &           1.299 $\pm$ 0.092 &          1.263 $\pm$ 0.322 &         0.935 $\pm$ 0.042 &         \textbf{0.779 $\pm$ 0.018 } \\
             & 4.0   &           0.946 $\pm$ 0.043 &          0.940 $\pm$ 0.100 &         0.807 $\pm$ 0.019 &         \textbf{0.746 $\pm$ 0.011 } \\
             & 8.0   &           0.847 $\pm$ 0.021 &          0.839 $\pm$ 0.035 &         0.751 $\pm$ 0.013 &         \textbf{0.733 $\pm$ 0.014 } \\
             & \textit{Non-Private} &           \textbf{0.622 $\pm$ 0.013} &          0.759 $\pm$ 0.015 &         0.681 $\pm$ 0.012 &         0.681 $\pm$ 0.012 \\
\bottomrule
\end{tabular}

\end{sc}
\end{small}
\end{center}
\vskip 0.1in
\end{table*} 
\begin{figure*}
\centering
    \includegraphics[width=\textwidth]{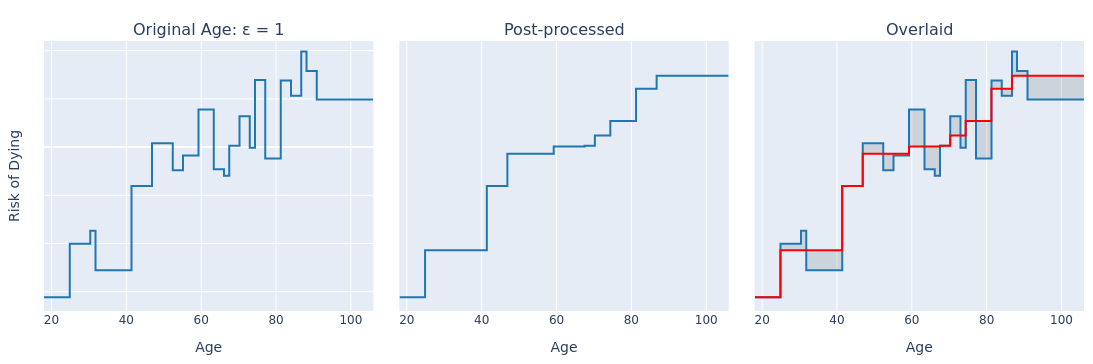}
  \caption{
  (Left) Original learned $f_{age}$ from the healthcare dataset with $\eps = 1$. 
  (Middle, Right) Postprocessed to enforce monotonicity.}
  \label{fig:postprocessed}
\end{figure*}

\section{Discussion}

\subsection{Editing DP-EBM Models}

While this paper has primarily focused on introducing and comparing DP-EBMs in terms of standard machine learning metrics, we believe it is important to highlight the unique capabilities that arise when combining interpretability with differential privacy. For example, recent work has shown that adding differentially private noise to machine learning algorithms can disproportionately impact minority groups \cite{cummings2019compatibility, bagdasaryan2019differential}. This concern is compounded when models are trained and deployed in high-risk domains like finance and healthcare -- even small errors on sparse regions of the feature space can have disatrous consequences. 

With intelligible models like DP-EBMs, some effects of noise on predictions are visible. 
Figure \ref{fig:age-compare} shows the shape function DP-EBMs learned for the ``Age" feature in the healthcare dataset at 3 different levels of privacy. In many healthcare problems, risk should monotonically increase as a function of age. While the non-private and $\eps = 4$ models learn this behavior, there is significant distortion as a result of differentially private noise at $\eps = 1$. In this example, the $\eps = 1$ model suggests that patients who are 80 are considerably lower risk than those who are 77 or 82, which is almost certainly not a real signal in the dataset. By using an intelligible model, domain experts can inspect the shape function $f_{age}$ and prevent deploying a risky model that under predicts on 80 year olds.

In addition to catching errors, domain experts can also correct unwanted learned effects without impacting privacy. 
Because the shape function $f_{age}$ exactly describes how a model makes predictions, users can edit graphs of any feature to change the model. In our example, modifying y-axis value for $f_{age}$ at ages 78-81 to remove the unwanted blip would remove this noise bias from the model. 
This form of model editing uses no data, and therefore results in no additional privacy loss under the post-processing property of differential privacy \cite{dwork2006calibrating}. The ability to safely inspect and edit DP-EBM models before deployment is important for creating trust in differentially private models in high risk situations because some of the impacts of noise can be corrected.

\subsection{Constraints such as monotonicity}

More complex forms of editing are also possible. For example, we can ensure monotonicity across the entire feature by borrowing from the calibration literature and applying \emph{isotonic regression} on the shape function \cite{chakravarti1989isotonic}. 
Isotonic regression uses the Pool Adjacent Violators (PAV) algorithm to minimize the edits necessary to ensure monotonicity, and is optimal with respect to squared error on the differences between the two functions \cite{vaneeden1958testing}. Importantly, this process only uses public information from DP-EBMs -- the learned shape functions $f_k$, and the histogram definition $\hat{H}_k$. 

Figure \ref{fig:postprocessed} shows the effects of applying isotonic regression to the noisy $f_{age}$. 
While enforcing monotonicity is possible with models more expressive than DP-EBMs, this typically requires additional constraints during training and may consume more privacy budget and complicate the privacy analysis. 
It is a nice advantage of DP-EBMs that monotonicity can be achieved cleanly via post-processing. %

\subsection{Differential Privacy as a Regularizer}
Figures \ref{fig:age-compare} and \ref{fig:adult_all} also show that adding modest amounts of differentially private noise, like $\eps = 4$, can act as a regularizer to the model. 
The rise in risk between ages 50 and 90 here is smoother, whereas the non-private version learns a ``jumpier" function. 
Smoothness is traditionally a difficult property to achieve with tree-based boosted GAMs. 
The non-private EBM algorithm typically wraps the training process in multiple iterations of bagging to make graphs smoother \cite{lou2012intelligible, caruana2015intelligible}. 
Our experiments suggest that modest amounts of differentially private noise might act as an effective regularization tool.

The relationship between smoothness and interpretability is complex: smooth graphs may be easier to interpret, but over-regularization can hide real signals in the data. 
The use of differential privacy as a regularizer is well known \cite{chaudhuri2011differentially}. Our paper visibly reinforces this notion through intelligibility, and we believe studying this effect further on GAMs might be interesting future research.

\section{Conclusion}
We present DP-EBMs, a differentially private learning algorithm for GAMs which achieves remarkably high accuracy and interpretability with strong privacy guarantees.
Our empirical evaluations show that DP-EBMs outperform other differentially private learning algorithms for both classification and regression on tabular datasets. Beyond just accuracy, we also show how interpretability can complement differential privacy by enabling users to uncover undesirable effects of noise, edit unwanted bias out of their models, and enforce desirable constraints like monotonicity with no additional privacy loss. These practical advantages might represent an important step forward for enabling the use of differentially private models in industries like healthcare, finance, and criminal justice. 

\begin{figure*}[b]
\centering
    \includegraphics[width=\textwidth]{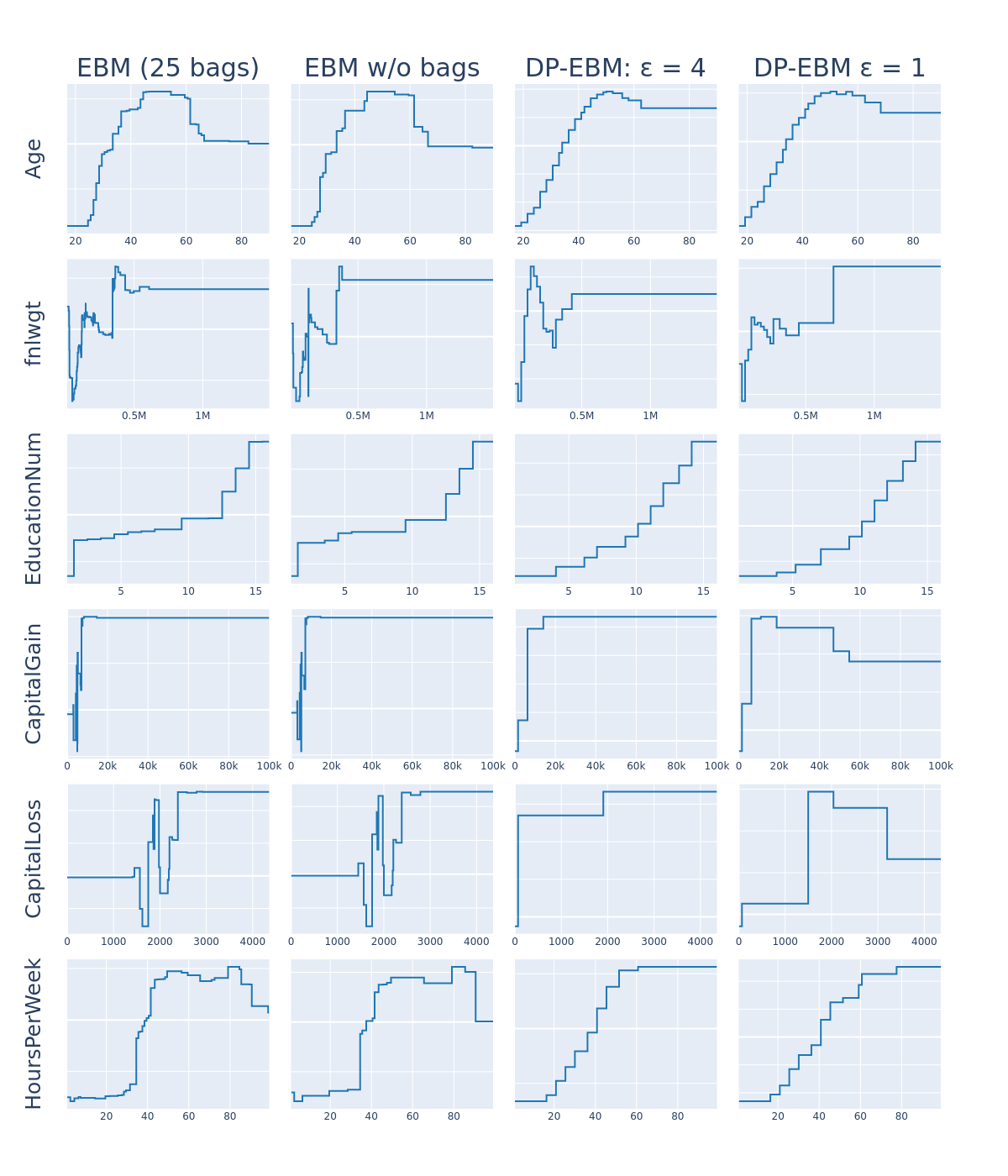}
  \caption{Shape function comparisons for all numeric features in the Adult Income dataset. We include the standard EBMs wrapped in 25 layers of bagging, EBMs without bagging, and two DP-EBMs at different privacy levels ($\eps=4$ and $\eps=1$). As expected, adding Gaussian DP noise acts as a strong regularizer --- the graphs on the right are smoother than those on the left.  In some cases this regularization is too strong, yet in other cases such as \textit{fnlwgt} the extra regularization might actually reduce model variance and improve intelligibility.}
  \label{fig:adult_all}
\end{figure*}

\clearpage

\section{Acknowledgement}
We would like to thank Paul Koch, Scott Lundberg, Samuel Jenkins, and Joshua Allen for their thoughtful discussions and copyediting.

\nocite{bu2020deep}
\nocite{plotly} 
\nocite{harris2020array} 
\nocite{scikit-learn, sklearn_api} 
\nocite{2020SciPy-NMeth}

\bibliography{references}
\bibliographystyle{icml2021}

\clearpage
\section{Appendix}

\subsection{Differentially Private Binning}

While not a main contribution of this paper, for completeness (and reproducibility) we describe the differentially private binning algorithm used as a preprocessing step to DP-EBMs. Our goal is to create bins for each feature such that each bin contains roughly equal proportions of data. For example, if the goal is to end up with 10 bins per feature, we expect each bin to roughly contain $\frac{1}{10}$ of the data.

Like other DP tree implementations, we assume that the min and max of each feature are supplied by the user \cite{dwork2014algorithmic, jagannathan2009practical}. The algorithm begins by uniformly dividing the feature space into \emph{equal width} bins, without looking at any user data. If the user requests $m$ bins in total, the binning procedure creates $2\cdot m$ equal width bins to begin with. We then create a differentially private histogram based on those uniform bin widths, adding noise with sensitivity 1 \cite{dwork2006calibrating}. Theorems \ref{thm:GDPComp} and \ref{thmGDPtoDP} are also applied here to track cumulative privacy budget across all $K$ features and calibrate how much noise to add. Finally, to transform the noisy equal width bins into equal density bins, the algorithm greedily post-processes the released bin definitions by collapsing small bins into their neighbors until a sufficiently large "quantile" bin is achieved. While this method can be sub-optimal on highly skewed distributions, we find that it works well in practice on most datasets, and users have some control by choosing appropriate min/max values or applying transforms to features prior to training. The full algorithm is detailed below.

\begin{algorithm}[]
   \caption{Differentially Private Quantile Binning}
   \label{alg:dp-bin}
\begin{algorithmic}
   \STATE {\bfseries Input:} data $X$, target bins $m$, privacy parameters $\eps, \delta$
   \STATE {\bfseries Output:} Histogram $\hat{H}_k$ per feature
   \STATE Target datapoints per bin: $t = \frac{|X|}{m}$
   \FOR{feature $0...K$} 
    \STATE Equal width bins: $H_k = Hist(X[:, k], 2m)$
    \STATE Add noise to counts: $\hat{H}_k = H_k + \sigma \cdot N(0,1)$
    \STATE Postprocess:
    \FOR{each histogram bin {$b_i \in \hat{H}_k$}}
        \IF{$|b_i| < t$}
            \STATE Greedily collapse bins: $b_{i + 1} = b_{i+1} + b_i$
            \STATE Delete previous bin: $Delete(\hat{H}_k, b_i)$
        \ENDIF
    \ENDFOR
    \STATE Check if final bin is sufficiently large
    \IF {$|b_i| < t$}
        \STATE Collapse into previous bin: $b_{i - 1} = b_{i - 1} + b_i$
        \STATE Delete final bin: $Delete(\hat{H}_k, b_i)$
    \ENDIF
   \ENDFOR
\end{algorithmic}
\end{algorithm}

\subsection{Adaptations to other settings}

Algorithm \ref{alg:dp-ebm} in the main body of the paper focuses on  regression, which often can be adapted to other settings. It also is possible to use many alternative loss functions in DP-EBMs with no change to the privacy analysis. For example, to adapt DP-EBMs to binary classification, we might prefer residuals to be \emph{logits}. Our proof of privacy depends on ensuring that the sensitivity of the sum of residuals is bounded by at most $R$ at each iteration. In the regression setting, we show that the sum of residuals $T$ can be framed as:
\vspace*{0.1in}

\centerline{$T = \eta \cdot \left(\sum_{b \in S_{\ell}} \sum_{x_i \in \hat{H}_k (b)} y_i \right)  - Z\\ $}
where $Z$ is entirely public information released from previous iterations of the model. Therefore, simple transformations on Z do not affect the sensitivity of each update or the ultimate privacy guarantee of the algorithm. For binary classification, the only modification is to line 25:
\vspace*{0.1in} \textcolor{white}{.}
\centerline{$r^t_i = y_i + \frac{1}{1 + e^{\sum{f^t_k(\rho(H_k,x_i))}}} - 1$}

\subsection{Abnormally low AUROCs}

In our experiments, some models produce AUROCs substantially lower than 0.5. For example, on the credit fraud dataset for $\varepsilon=1$, DPBoost had an average test set AUROC of 0.438. This was a bit surprising, as predicting random labels should result in AUCs near 0.5. Further investigation suggests that the noise DP adds to models increases the variance of model predictions enough to make AUCs much larger and smaller than 0.5.

For example, Figure 5 shows the distribution in AUCs for the healthcare dataset when predictions are made completely randomly (left), and also shows the distribution for the same dataset when predictions are made with DP Logistic Regression with $\varepsilon=0.5$ (right).  Both distributions have mean 0.5, but the distribution is much wider for the model with $\varepsilon=0.5$.  Similar behavior is observed for all DP algorithms, including DP-EBMs with $\varepsilon < 0.1$.

\vspace*{0.1in}
\begin{figure}[H]
    \centering
    \includegraphics[width=230px]{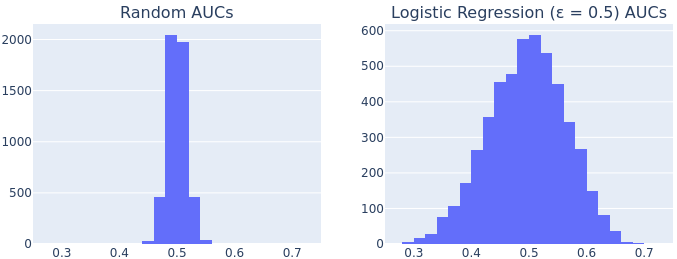}
    \caption{Distribution of AUROCs on 5000 train/test splits of healthcare dataset. Left: Randomly generated predictions.
    \\\hspace{\textwidth}
    Right: DP Logistic Regression ($\varepsilon=0.5$) test set AUROCs.}
    \label{fig:my_label}
\end{figure}

 \end{document}